\documentclass[letterpaper]{article}
\usepackage{aaai}
\usepackage{microtype}
\usepackage{times}
\usepackage{helvet}
\usepackage{courier}
\usepackage{graphicx} 
\usepackage{CJKutf8}
\usepackage{booktabs}
\usepackage{color}
\usepackage{graphicx}
\usepackage{amsmath}
\usepackage{amsthm}
\usepackage{enumitem}
\usepackage{multirow}
\usepackage{booktabs}
\usepackage{natbib}
\usepackage{graphicx}
\usepackage{siunitx}
\theoremstyle{definition}
\newtheorem{theorem}{Theorem}

\newcommand{\tabincell}[2]{\begin{tabular}{@{}#1@{}}#2\end{tabular}}  
\usepackage{url}
\usepackage{breakurl}
\usepackage[breaklinks]{hyperref}

\usepackage{amsfonts}
\usepackage{amssymb}
\usepackage{bbm}

\DeclareMathOperator*{\argmax}{argmax}

\begin{document}
	\begin{CJK}{UTF8}{gkai}
		%
		\title{A Multi-Agent Reinforcement Learning Method for Impression Allocation \\in Online Display Advertising}
		
		\author{Di Wu, Cheng Chen, Xun Yang, Xiujun Chen, Qing Tan, Jian Xu, Kun Gai\\
			Alibaba Group\\
			Beijing, P.R.China\\
\{di.wudi,chencheng.cc,vincent.yx,xiujun.cxj,qing.tan,xiyu.xj,jingshi.gk\}@alibaba-inc.com\\
		}
		
		\maketitle
		\begin{abstract}
			\begin{quote}
				 In online display advertising, guaranteed contracts and real-time bidding (RTB) are two major ways to sell impressions for a publisher. Despite the increasing popularity of RTB, there is still half of online display advertising revenue generated from guaranteed contracts. Therefore, simultaneously selling impressions through both guaranteed contracts and RTB is a straightforward choice for a publisher to maximize its \emph{yield}. However, deriving the optimal strategy to allocate impressions is not a trivial task, especially when the environment is unstable in real-world applications. In this paper, we formulate the impression allocation problem as an auction problem where each contract can submit virtual bids for individual impressions. With this formulation, we derive the optimal impression allocation strategy by solving the optimal bidding functions for contracts. Since the bids from contracts are decided by the publisher, we propose a multi-agent reinforcement learning (MARL) approach to derive cooperative policies for the publisher to maximize its \emph{yield} in an unstable environment. The proposed approach also resolves the common challenges in MARL such as input dimension explosion, reward credit assignment, and non-stationary environment. Experimental evaluations on large-scale real datasets demonstrate the effectiveness of our approach.
			\end{quote}
		\end{abstract}
		
		\section{Introduction}
		
		In recent years, online display advertising has become one of the most influential businesses with \$39.4 billion revenue\footnote{Display-related ad formats include: Banner and Video.} for FY 2017 in US alone \citep{iab-hy-2017}. 
		As shown in Fig. \ref{fg:displayads}, typically when a user visits a publisher, e.g., a news website, there would be one or more ad impression opportunities generated in real time. Advertisers are able to acquire these opportunities to display their ads at certain costs and these costs eventually become the revenue of the publisher.
		
		For a publisher, there are two major ways to sell impressions in the field of display advertising. The first one is through \emph{guaranteed contracts} (also referred as \emph{guaranteed delivery} \citep{chen2014dynamic}). A guaranteed contract is an agreement between an advertiser and a publisher by negotiating directly or by going through a programmatic guaranteed mechanism \citep{chencombining}. The contract usually specifies the contract payment amount, the campaign duration and the desired number of ad impressions. The advertiser typically makes the payment before the ad delivery starts and the publisher guarantees the desired number of ad impressions. The publisher is also responsible for any shortfall in the number of impressions delivered. A penalty is usually incurred based on the volume of under-delivery.
		
		The second way to sell impressions is through \emph{real-time bidding} (RTB). RTB allows advertisers to bid in real-time for impressions and does not guarantee the impression volume for any advertiser \citep{yuan2013real}. For each impression opportunity, the advertiser offering the highest bid wins the opportunity to display its ad. The cost of the winner is determined based on the auction mechanism. In this paper, without loss of generality, we focus our discussion under the second price auction \citep{edelman2007internet} where the winning advertiser is charged the second highest bid in the auction. Our approach is also applicable under other auction mechanisms such as Vickrey–Clarke–Groves (VCG) \citep{nisan2007algorithmic}.
		
		\begin{figure} 
			\centering
			\includegraphics[width=0.3\textwidth]{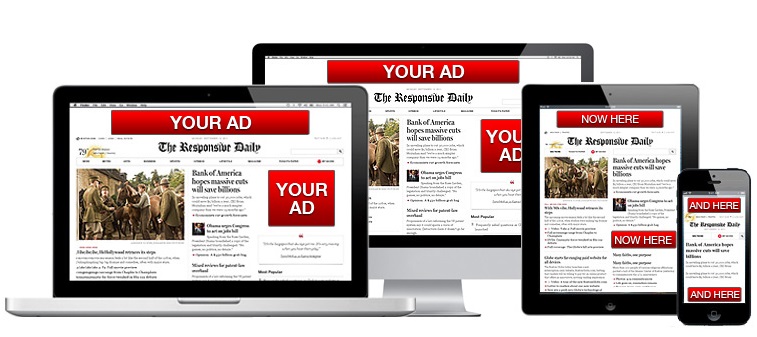}
			\caption{An example of online display ads. Advertisers can show their ads on publishers' websites and/or apps to the targeted audience with certain cost (revenue to the publisher).}
			\label{fg:displayads}
		\end{figure}
		
		Despite the increasing popularity of RTB, there is still half of the online display advertising revenue generated from guaranteed contracts according to \citep{FisherPG}. For a publisher, simultaneously selling impressions through both guaranteed contracts and RTB is a straightforward choice in terms of maximizing its \emph{yield}. The challenge turns out to be how the impressions should be allocated when guaranteed contracts and RTB are both acquiring impressions. In other words, we are interested in deriving the optimal impression allocation strategy for the publisher. Generally speaking, there are three main considerations when a publisher strategically allocates its impressions in this scenario.
		
    	\begin{itemize}
          \item \textbf{Maximizing total revenue}: The revenue is associated with guaranteed contracts, the revenue from RTB, and the potential contract violation penalties.
          \item \textbf{Fulfilling guaranteed contracts}: Under-delivery is problematic since it may be harmful for the credit of the publisher. It is common to negotiate proper violation penalties to resolve this challenge.
          \item \textbf{Ensuring impression quality for guaranteed contracts}: Advertisers with guaranteed contracts are more and more concerned with the impression quality. Quality of the delivered impressions is crucial to ensure the long-term value of these advertisers.
        \end{itemize}
		
		The above considerations are common in real-world scenarios. Therefore, deriving the optimal impression allocation strategy is not a trivial task. What makes it even more challenging is the instability of the environment. When deriving the optimal impression allocation strategy, the stationary assumptions presented in the work by \citep{li2016optimal,jauvion2018optimal,balseiro2014yield} are difficult to satisfy in the real-world scenarios. First, the environment stability is vulnerable to unexpected traffic changes such as those brought by sales events on holidays. Second, usually concurrent with the traffic changes, the market price distribution of the impressions can also deviate from the empirical. Finally, the unpredictable advertiser behaviors in RTB including modifying budget, bid, and targeted audience can make the environment more complicated and dynamic \citep{wu2018budget}.
		
		To derive an optimal impression allocation strategy that takes into account the above-mentioned considerations and challenges, we propose to analyze the problem from a novel perspective. Since the allocation is non-trivial and the environment is highly dynamic, can the guaranteed contracts also participate in the real-time auctions so that they can also enjoy the liquidity and the impressions can be fully auctioned?  More specifically, can each guaranteed contract be treated as a bidding agent which is able to submit bids for individual impressions and the impression allocation is based on the submitted bids from both guaranteed contracts and RTB? We will show that such a setup can actually lead us to the optimal impression allocation strategy. And the optimal strategy also possesses the appealing property that the optimal bid from a guaranteed contract only depends on the impression quality and the under-delivery penalty, regardless of the RTB status. 
		
		It is worth noting that the bidding agents representing the guaranteed contracts are essentially different from those representing RTB ads. Their submitted bids are actually decided by the publisher. In other words, they are able to cooperate with each other to achieve the optimum and this is one of the key ideas in this paper. Since they submit bids in real-time to participate in auctions in an unstable environment, it is natural to employ the \emph{multi-agent reinforcement learning} (MARL) approach to learn the optimal strategy. We follow the work in \citep{lowe2017multi} and propose the \textit{multi-agent policy optimization with local observations} (MAPOLO) approach to boost the convergence by shaping the reward function \citep{wu2018budget}. With MAPOLO, the common challenges in MARL are effectively addressed, such as input dimension explosion\footnote{The critic will receive the observations and actions from all agents as input in \citep{lowe2017multi}.}, reward credit assignment \citep{foerster2017counterfactual}, and non-stationary environment caused by agent policy changes \citep{tesauro2004extending}. 
		
		To evaluate the effectiveness of our approach, we conducted experiments on large-scale real-world datasets. Compared with some of the state-of-the-art methods, we observed substantial improvements on both impression allocation results and MARL performance metrics. Our main contributions can be summarized as follows:
		
		\begin{enumerate}
			\item We study the problem of optimal impression allocation in display advertising with both guaranteed contracts and RTB. A novel and efficient allocation strategy is proposed and we prove that it is essentially optimal.
			\item We devise a MARL approach to achieve the optimum in the non-stationary and highly dynamic environment. To the best of our knowledge, this is the first work of applying MARL in impression allocation.
			\item Our MARL approach also addresses some common challenges in MARL including input dimension explosion, reward credit assignment and non-stationary environment caused by the change of agent policy.
		\end{enumerate}
	
	The rest of this paper is organized as follows. The optimal impression allocation strategy is derived in section 2. In section 3, we present our MARL approach to learn the optimal strategy and demonstrate its merits in dealing with some common MARL challenges. Experimental evaluation results are shown in section 4, followed by the related work in section 5. We conclude the paper in section 6. 
		
		
	\section{Optimal Impression Allocation}
		
		%
		
		For the publishers who sell impressions through both guaranteed contracts and RTB, one of their impression allocation motivations is to maximize the total revenue from both contracts and RTB. Meanwhile, the impression quality\footnote{We abuse the concept \emph{impression quality} a little bit here so that it reflects both contract fulfillment and average impression quality.} of the contracts can affect the satisfaction of the contract advertisers and therefore affect the long-term revenue. The ultimate goal of impression allocation is to simultaneously maximize RTB revenue, contract revenue, and contract impression quality.

		\subsection{Problem Formulation}
		Suppose there are $n$ impressions indexed by $i$ to be allocated by the publisher. On the one hand, suppose that there are $m$ guaranteed contracts indexed by $j$ to be served. For each contract $j$, let $d_j$ be the demand impression volume and $c_j$ be the unit price of each impression so that the contract value is $c_jd_j$. Suppose the contract violation penalty of contract $j$ for each impression is $p_j$, that is, if the impressions served to contract $j$ is less than $d_j$, the publisher will have to be responsible for the penalty at $p_j$ each impression. On the other hand, for each impression $i$, RTB will also provide a list of bids of which we are mostly interested in the first and second highest bids $\mathbbm{b}_{i1}$ and $\mathbbm{b}_{i2}$ under the second-price auction mechanism. If the impression is allocated to RTB, then the publisher will earn $\mathbbm{b}_{i2}$. Let $x_{ij}$ be the binary indicator whether impression $i$ is allocated to contract $j$ and obviously if $\sum_j{x_{ij}=0}$ then the impression is allocated to RTB, we are able to derive revenues from both guaranteed contracts and RTB. More specifically, the revenue from guaranteed contracts is $R_{GC} = \Sigma_j c_jd_j - \Sigma_j p_jy_j$ where $y_j=d_j-\Sigma_i x_{ij}$ is the impression under-delivery amount and the revenue from RTB is $R_{RTB}=\Sigma_i (1-\Sigma_j x_{ij})\mathbbm{b}_{i2}$. Since we also care about the impression quality delivered to guaranteed contracts, let $q_{ij}$ be impression $i$'s quality for contract $j$ and $\lambda_j$ be the quality weight, the total contract impression quality can be denoted by $Q_{GC}=\Sigma_{ij} \lambda_jx_{ij}q_{ij}$. As mentioned above, our goal of impression allocation is to maximize both revenue $R_{GC} + R_{RTB}$ and quality $Q_{GC}$ for short-term and long-term benefits. By putting these objectives together, we define $yield$ as $R_{GC} + R_{RTB} + Q_{GC}$.  The optimal impression allocation problem can be formulated as follows\footnote{For simplicity, we take out the binary constraints of $x_{ij}$ in the formulation \eqref{lp1}. It can be proved that the influence of this simplification is negligible since in the optimal solution of \eqref{lp1}, the numbers of non-binary $x_{ij}^*$ are much fewer than that of binary $x_{ij}^*$.}:

		
		\begin{align}
		\small
		\underset{x_{ij},y_j}{\text{maximize}}  & & R_{GC} + R_{RTB} + Q_{GC} & & & \tag{LP1} \label{lp1}\\
		\text{s.t.} & & \Sigma_i x_{ij} + y_{j} \ge d_j,\quad &&&\forall j, \nonumber\\
		& & \Sigma_j x_{ij} \le 1,\quad &&& \forall i, \nonumber\\
		& & x_{ij} \ge 0,\quad &&& \forall i,j, \nonumber\\
		& & y_{j} \ge 0,\quad &&& \forall j. \nonumber
		\end{align}
		
		\begin{figure}
			\centering
			\includegraphics[width=0.20\textwidth]{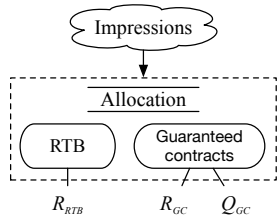}
			\caption{Impression allocation process of a publisher. Publishers aim to maximize the sum of $R_{RTB}$, $R_{GC}$ and $Q_{GC}$.}
			\label{fg:hy_method_framework}
		\end{figure}
		
		\subsection{The Optimal Allocation Strategy}
		
		Deriving the optimal solution to the problem is not straightforward. We propose to look at the problem from a different perspective. That is, we are interested in whether the guaranteed contracts can also participate in the real-time auctions so that they can also enjoy the liquidity and the impressions can be fully auctioned. More specifically, can each guaranteed contract be treated as a bidding agent which is able to submit bids for impressions and the impression allocation is based on the submitted bids from both guaranteed contracts and RTB? We will show that such a setup can actually lead us to the optimal allocation strategy. 
		
		\begin{table}
			\caption{Notations for impression allocation.}
			\centering
			\scalebox{0.64}{
				\begin{tabular}{c|l}
					\toprule
					Notations  &  Descriptions  \\  
					\hline
					$d_j$	& Impression demand of contract $j$.  \\
					$c_j$	& Payment of one impression for contract $j$.  \\
					$p_j$	& Contract violation penalty for one impression of contract $j$. \\
					$y_j$	& The shortfall of contract $j$.  \\
					$x_{ij}$	& Binary binary indicator whether impression $i$ is allocated to contract $j$.  \\
					$q_{ij}$	& Impression $i$'s quality for contract $j$.  \\
					$\lambda_j$	& The impression quality weight for contract $j$.  \\
					$b_i^*$		&  Maximal bid of all guaranteed contracts for impression $i$.  \\
					$\mathbbm{b}_{i1}$	& The highest bid of impression $i$ in RTB.  \\
					$\mathbbm{b}_{i2}$	& The second highest bid of impression $i$ in RTB.  \\
					$Q_{GC}$ 	& The impression quality of all guaranteed contract. \\
					$R_{GC}$	& Total revenue of guaranteed contracts.  \\
					$R_{RTB}$	& Total revenue of RTB.  \\
					\textit{yield} & $R_{GC} + R_{RTB} + Q_{GC}$ \\
					\bottomrule
				\end{tabular}
			}
			\label{table:notation1}
		\end{table}
		
		\begin{theorem} 
			Suppose for each impression $i$, every contract $j$ submits a bid $b_{ij}$. Let $b^*_i = \max_j b_{ij}$ and $j^*=\argmax_j b_{ij}$. Consider the allocation strategy that allocates impression $i$ to contract $j^*$ if $b^*_i  > \mathbbm{b}_{i2}$ and otherwise to RTB. The strategy solves the optimization problem defined in \eqref{lp1} if the bid from the guaranteed contract $j$ takes the form:
			\label{algo:prove}
		\end{theorem}
		
		\begin{equation} \label{eq:optimal_f}
		b_{ij} = \lambda_j\ q_{ij} + \alpha_{j}
		\end{equation}
		
		\noindent
		, where $\alpha_{j} \in [0, p_j]$.

		\begin{proof}
			 The maximal revenue from RTB $\Sigma_i\mathbbm{b}_{i2}$, along with the total payment from contracts $\Sigma_jc_jd_j$ can be considered as constants. Then, $\text{maximize} \ \Sigma_{ij}(1-x_{ij})\mathbbm{b}_{i2} + \Sigma_jc_jd_j - \Sigma_j y_j p_j + \Sigma_{ij} \lambda_jx_{ij}q_{ij}$ can be simplified as $\text{maximize}\ -\Sigma_{ij}x_{ij}\mathbbm{b}_{i2} - \Sigma_j y_j p_j + \Sigma_{ij} \lambda_jx_{ij}q_{ij}$. Thus, the dual problem of \eqref{lp1} is as follows:
			
			\begin{align}
			\small
			\underset{\alpha_j,\beta_i}{\text{minimize}} & & \Sigma_{i}\beta_{i} -\Sigma_{j} \alpha_j d_{j} & & & \tag{LP2}\label{lp2} \\
			\textup{s.t.} &  & \lambda_j q_{ij} + \alpha_{j}  - \mathbbm{b}_{i2} \le \beta_i,  & & &  \forall i,j, \label{lp2_1} \\
			& & \alpha_j \le p_{j},  & & &  \forall j, \nonumber\\ 
			& & \alpha_j \ge 0,  & & & \forall j, \nonumber\\
			& & \beta_i \ge 0,  & & & \forall i. \nonumber
			\end{align}
			
			Suppose the optimal solution to \eqref{lp1} and \eqref{lp2} are $x^*_{ij}, y_j^*$ and $\beta^*_i, \alpha^*_j$ respectively. We denote impression $i$'s bid of contract $j$ as $b_{ij} = \lambda_j q_{ij} + \alpha_{j}^*$. According to the complementary slackness theorem,  we have
			
			\begin{align}
			\small
			x_{ij}^* \ (b_{ij} - \mathbbm{b}_{i2} - \beta_i^*) &= 0,     & & \forall i,j, \label{dual1}\\ 
			(\Sigma_j {x_{ij}^*} -1) \ \beta_i^* &= 0, & & \forall i. \label{dual2}
			\end{align}
			
			\noindent
			$\forall i=1,2,...,n; j=1,2,...,m$:
			\begin{itemize}
				\item If $b_{ij} < \mathbbm{b}_{i2}$ then $b_{ij} - \mathbbm{b}_{i2} - \beta_{i}^* < 0$. Based on Eq. \eqref{dual1} we can infer that $x_{ij}^* = 0$, which means impression $i$ is allocated to RTB.
				
				\item If $b_{ij} > \mathbbm{b}_{i2}$, we can infer that $\beta_i^* > 0$ and $\sum_j{x_{ij}^* = 1}$ according to Eqs. \eqref{lp2_1},\eqref{dual2}, which means impression $i$ is allocated to guaranteed contracts. Let $x_{ik}^* > 0$, then we have $b_{ik} = \mathbbm{b}_{i2} + \beta_{i}^*$ based on Eq. \eqref{dual1}. Therefore $b_{ij} \le \mathbbm{b}_{i2} + \beta_{i}^* = b_{ik}$ , i.e., $k = \argmax_j {b_{ij}}$. This means the impression is assigned to the contract with maximal bid.
			\end{itemize}
		\end{proof}
		
		
\section{A Multi-Agent Reinforcement Learning Approach}
		
		As discussed in the last section, all guaranteed contracts are considered as bidders and their optimal bidding functions are given by Theorem \ref{eq:optimal_f}. However, since they submit bids in an unstable auction environment, as  argued in section 1, the bids may deviate from the optimal one during the real-time impression allocation process.
		Note that the bids submitted by contract bidders are essentially determined by the optimal impression allocation strategy, it is natural to employ the MARL approach to coordinate the adjustments of contract bidding functions close to the optimal ones as much as possible. In the rest of this section, we first introduce the MARL preliminaries, and then present our MARL approach MAPOLO. The frequently used notations for MARL in this paper are summarized in Table \ref{table:notation2}.
		
		\subsection{MARL Preliminaries}
		
		Reinforcement learning (RL) is a machine learning approach inspired by behaviorist psychology. In RL, an agent interacts with environment by sequentially taking actions, observing consequences, and altering its behaviors in order to maximize a cumulative reward. Multi-agent reinforcement learning (MARL) is an extension of RL and usually formalized in a Markov game framework \citep{littman1994markov}. For a Markov game with $m$ agents, it consists of: a state space $\mathcal{S}=\{s\}$, an collection of action spaces $\mathcal{A}_1, ..., \mathcal{A}_m$, state transition dynamics $\mathcal{T}: \mathcal{S}\times\mathcal{A}_1\times ... \times \mathcal{A}_m\rightarrow\mathcal{P}(\mathcal{S})$ where $\mathcal{P}(\mathcal{S})$ is the set of probability measures on $\mathcal{S}$, an immediate reward function $r_j: \mathcal{S}\times\mathcal{A}_1\times ... \times \mathcal{A}_m \rightarrow \mathbb{R}$ for agent $j$, and a discount factor $\gamma\in [0,1]$. A policy, denoted by $\pi_j: \mathcal{O}_j\rightarrow\mathcal{P}(\mathcal{A}_j)$, where $\mathcal{O}_j$ is the partial observation of state $\mathcal{S}$ for agent $j$ and $\mathcal{P}(\mathcal{A}_j)$ is the set of probability measures on $\mathcal{A}_j$. 
		Each agent uses its policy to interact with the environment and gives a trajectory of states, actions, and rewards $\{s_1, (a_{1}^{(1)}, ..., a_{m}^{(1)}), (r_{1}^{(1)}, ..., r_{m}^{(1)}), ..., s_t, (a_{1}^{(t)}, ..., a_{m}^{(t)}), (r_{1}^{(t)}\\, ..., r_{m}^{(t)})\}$, where $t \in [1, ..., T]$. For simplicity, we shorten the element of the trajectory at step $t$ as $(s_t, \textbf{a}_t, \textbf{r}_t)$. 
		The objective of agent $j$ is to find a policy $\pi_j^*$ to maximize its expected sum of discounted rewards.
		
		\begin{equation} \label{eq:policy}
		\begin{aligned}
		\pi^*_j = \argmax_{\pi_j} \mathbb{E}[\Sigma_{t=1}^T\gamma^{t-1}r_{j}^{(t)}|\pi_j]
		\end{aligned}
		\end{equation}
		
		
		\subsection{Modeling}
		
		Back to our problem, based on the optimal solution derived in last section, we consider each contract as an agent and all agents should cooperate with each other to regulate their $\alpha_j$ close to the optimal ones as much as possible based on the current state of the environment. 
		For agent $j \in [1, 2, ..., m]$, we consider an episodic Markov game with discount factor $\gamma=1$ where a deterministic episode (typically one day) starts with contract demand amount $d_j$, an initial bid parameter $\alpha_{j}^{(0)}$, and a contract violation penalty $p_j$ per impression. 
		As shown in Fig. \ref{fg:env_agent_interaction}, agents regulate their $\alpha_j$ sequentially with a fixed number of $T$ steps (typically 15 minutes between two consecutive steps) until the episode ends. At each time step $t \in [1, ...,T]$, agent $j$ gets observation $o_j^{(t)}$ based on state $s_t$ and takes an action $a_j^{(t)}$ to adjust $\alpha_j^{(t-1)}$ to $\alpha_j^{(t)}$, and then receives an immediate reward $r_j^{(t)}$ which is equaling the \textit{yield} of impressions between time-step $t$ and $t+1$. Agent $j$'s bid for any impression $i$ between time step $t$ and $t + 1$ is decided by the optimal bidding function \eqref{eq:optimal_f}, and the goal of each agent $j$ is to learn a $\alpha_j$ control policy to maximize the cumulative reward $R_j=\Sigma_{t=1}^T\gamma^{t-1}r_{j}^{(t)}$. More specifically, the core elements of the Markov game are further explained as follows:
		
		\begin{table}
			\caption{Notations for MARL modeling.}
			\centering
			\scalebox{0.6}{
				\begin{tabular}{c|l}
					\toprule
					Notations  &  Descriptions  \\  
					\hline
					$T$	& The episode length. \\
					$s_t$	& The state of environment at time-step $t$.  \\
					$o_{j}^{(t)}$ & The observation of agent $j$ at time-step $t$.  \\
					$a_{j}^{(t)}$ & The action of agent $j$ at time-step $t$.  \\
					$r_{j}^{(t)}$ & The reward of agent $j$ at time-step $t$.  \\
					$\alpha_{j}^{(t)}$	& The parameter in Eq. \eqref{eq:optimal_f} of contract $j$ at time-step $t$.  \\
					$Q_{j}^{(t)}$ & Action-value function of agent $j$ at time-step $t$, equaling $\sum_{i=t}^T\gamma^{i-t}r_{j}^{(t)}$\\
					$R_j$	& The return of agent $j$, equaling $\sum_{t=1}^T\gamma^{t-1}r_{j}^{(t)}$.  \\
					$\pi_j^{\mathbbm{r}}(o_j)$	&  The policy of agent $j$ with reward function $\mathbbm{r}$, mapping $o_j$ to $a_j$. \\

					\bottomrule
				\end{tabular}
			}
			\label{table:notation2}
		\end{table}
		
		\begin{itemize}[align=left,leftmargin=0.2in]
			\item[$\mathcal{S}$:] The state $s_t$ should in principle reflect the contract status and the RTB environment, which mainly includes the following three parts: firstly, the time information, which tells the agent the current stage of the impression allocation process; secondly, the contract demand fulfillment status, such as left impression volume to satisfy and demand fulfilling speed; last, the reward $r_t$, which represents the \textit{yield} at time-step $t$. Note that the features in state could be adjusted to adapt to the specific scenarios.
			
			\item[$\mathcal{A}_j$:] Recall that the goal of agent $j$ is to adjust the parameter $\alpha_j$ in Eq. \eqref{eq:optimal_f} close to the optimal one. We define the actions as limited range of real numbers, typically takes the form of $\alpha_j^{(t)} = \alpha_j^{(t-1)}(1+a_j^{(t)})$.
			
			\item[$r_j^{(t)}$:] The immediate reward for agent $j$ at step $t$ is the \textit{yield} of the impressions between time-step $t$ and $t+1$. Note that the immediate reward for each agent at step $t$ is same in a Markov game.
			
			\item[$\mathcal{T}$:] In our scenario, since the state space is too large to be feasible for a model-based approach, we adopt the commonly used model-free approach in our MARL algorithms. Therefore, the transition dynamics are not explicitly modeled.
			
			\item[$\gamma$:] The reward discount factor $\gamma=1$ since the optimization goal of the impression allocation problem is to maximize total reward value regardless of the reward time. 
		\end{itemize}
		
		There are common challenges for MARL to work well in real applications. One of the most concerned challenges is the convergence efficiency. Usually, a MARL process does not converge efficiently due to the following three factors:
		
		\begin{itemize}
			\item Inefficient reward function: Reward function is critical to the convergence efficiency of RL algorithms \citep{Sutton2017RL}, and the reward provided by environment may not directly related (such as linearly related) to the goal of maximizing return \citep{wu2018budget}, which makes agents have to resort to complicated exploration strategies to improve performance.
			
			\item Non-stationary environment: From a single agent perspective, the environment would be non-stationary when the policies of other agents keep changing \citep{tesauro2004extending}, which results in fluctuating reward for a specific state-action pair. The non-stationary environment can significantly slow down the convergence process. 
			
			\item Input dimension explosion: The input dimension of MARL methods would explode when the number of agents increases \citep{lowe2017multi,foerster2017counterfactual}. As a result, the exploded dimension poses remarkable challenge to the convergence efficiency.
			
		\end{itemize}
		
		\begin{figure} 
			\centering
			\includegraphics[width=0.40\textwidth]{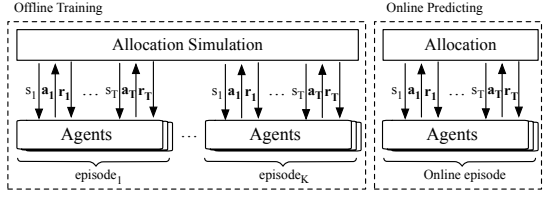}
			\caption{Illustration of MARL model training and predicting process. Each episode contains T steps and agents will be trained K episodes before being applied online.}
			\label{fg:env_agent_interaction}
		\end{figure}
		
		\subsubsection{MAPOLO}
		
		 We present the \emph{multi-agent policy optimization with local observations} (MAPOLO) approach to address all the challenges above. As shown in Fig. \ref{fg:MARLIA}, we adopt the multi-agent actor critic framework in \citep{lowe2017multi}, using a shaped reward function to coordinate agents' behaviors and allowing each agent only receive the local observations.
		
		The new reward function should be simple and direct, which means a better state-action pair (leading to a better return) should be given a larger reward. Meanwhile, we hope the reward for a specific state-action pair not be affected by other agents' behaviors unless they have achieved a better return, which will significantly relieve the challenge of non-stationary environment.
		Inspired by the work in \citep{wu2018budget}, we believe the return of an entire episode would be a feasible reward signal for all agents in MARL model. 
		The basic idea is that the larger the historical maximal return is, the larger reward should be provided to all the state-action pairs in the episodic trajectory, as shown in Eq. \eqref{eq:r_def}:
		
		\begin{equation} \label{eq:r_def}
		\begin{aligned}
		\mathbbm{r}(s, a_j^{(t)}) = \max_{e \in E(s, a_j^{(t)})} R_j^{(e)} \\
		\end{aligned}
		\end{equation}
		
		\noindent where $E(s, a_j^{(t)})$ represents the set of existing episodes that agent $j \in [1, 2, ..., m]$ took action $a_j^{(t)}$ at state $s$, and $R_j^{(e)}=\Sigma_{t=1}^T\gamma^{t-1}r_{j}^{(t)}$ is the original return for agent $j$ within episode $e$. The episodic nature of the process is leveraged so that new reward will be continuously updated during all policies' optimizations.
		
		\begin{figure} 
			\centering
			\includegraphics[width=0.25\textwidth]{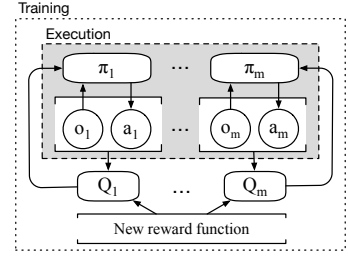}
			\caption{Illustration of MAPOLO. Agents are coordinated by the new reward function defined by \eqref{eq:r_def} during training and each agent only need to take local observations as input.}
			\label{fg:MARLIA}
		\end{figure}
		
		Another good property of the reward function is that the reward can be taken as coordinating information, indicating which action under the current state will benefit the final return.
		Therefore, we could just simplify the input of each agent as local observations to avoid the dimension explosion, which is a big challenge in existing MARL approaches\footnote{In the standard paradigm of centralized training with decentralized execution, the input of critic consists of all agents' observations.} \citep{lowe2017multi,foerster2017counterfactual}.
		Further, we prove that under some acceptable assumptions\footnote{The Markov game is deterministic and for each state there is only one optimal action. Here deterministic means taking actions $\textbf{a}_t$ at state $s_t$ will lead to $s_{t+1}$ with probability 1.}, the new reward defined in \eqref{eq:r_def} will lead to the optimal policy of the original Markov game.
		
		\begin{theorem} \label{prove2}
			
			Let $\pi^* = \{\pi^*_1,...,\pi^*_m\}$ be the optimal policy for the original Markov game,
			and $\pi^{\mathbbm{r}}=\{\pi^{\mathbbm{r}}_1,...,\pi^{\mathbbm{r}}_m\}$ be the optimal policy for the Markov game with new reward defined in Eq.\eqref{eq:r_def}. Then $\pi^{\mathbbm{r}} = \pi^*$ as long as all episodes start from the same initial state and for each state there is only one optimal action.
			
		\end{theorem}
		
		\begin{proof}

			According to Eq. \eqref{eq:policy} and the Markov property, the optimal policy at any state $s_t$ with observation $\{o_j^{(t)}|j = 1,...,m\}$ is as follows,
			\begin{equation}
			\small
			\begin{aligned}
			\pi^*_j & = \argmax_{\pi_j} [\Sigma_{i=1}^T \gamma^{-1}r_j^{(i)}|\pi_j], \forall j=1,...,m, \\
			\Rightarrow \pi^*_j(o_j^{(t)}) & = \argmax_{\pi_j} [\Sigma_{i=t}^T{\gamma^{i-t} r_j^{(i)}}|\pi_j], \forall j=1,...,m,
			\end{aligned}  
			\end{equation}
			which means that $\pi^*_j(o_j^{(t)})$ is the  optimal action that maximize the return of agent $j$. 
			
			Since the Markov game is deterministic and starts from the same initial state, according to the theorem proven in \citep{wu2018budget}, the action $\pi_j^{\mathbbm{r}}{(o_j^{(t)})}$ of each agent $j$ is also the optimal one that maximize the return for agent $j$, i.e.,
			\begin{equation}
			\small
			\pi_j^{\mathbbm{r}}{(o_j^{(t)})} = \argmax_{\pi_j} [\Sigma_{i=t}^T{\gamma^{i-t} r_j^{(i)}}|\pi_j], \forall j = 1,...,m.
			\end{equation}
			This means $\{\pi_1^{\mathbbm{r}}{(o_1^{(t)})},...,\pi_m^{\mathbbm{r}}{(o_m^{(t)})}\}$ is an optimal action for the multi-agent Markov game. Since the optimal action is unique, we have $\pi_j^{\mathbbm{r}}{(o_j^{(t)})} = \pi_j^*{(o_j^{(t)})}, \forall j=1,...,m, \forall t=1,...,T$. Therefore $\pi^{\mathbbm{r}} = \pi^*$.
			
		\end{proof}
		
		\section{Experimental Evaluation}
		
		
		We evaluate our proposed method from two aspects. First, we compare our approach with existing approaches and show its advantages in the optimal impression allocation tasks. Second, we demonstrate the desirable properties of our approach in dealing with the common challenges in MARL.
		
		\subsection{Experiment Setup}
		\subsubsection{Dataset}
		
		
		The experiment datasets are from a leading e-commerce advertising platform. We picked up five different publishers on this platform and for each publisher we extracted the real ad serving logs of two consecutive days in August, 2018. 
		The five datasets contain more than 35 millions impressions in total. Each dataset consists of the ad impressions, the guaranteed contracts, and the detailed RTB auction information\footnote{Due to confidential issues, the actual values related to revenue are transformed into meaningless numbers.}. 
        From each publisher's perspective, the environment such as the impression volume and RTB market price distribution changed significantly in these two days. The details of these datasets are shown in Table \ref{table:yield_cmp}. We use the data from the first day for training and those from the second day for testing.
		
		
		\subsubsection{Evaluation Metrics}
		
		The goal of the optimal impression allocation task is to maximize the \textit{yield}. Based on the optimal bidding equation defined in \eqref{eq:optimal_f}, the theoretically best \textit{yield} on the testing dataset can be obtained, denoted by $R^*$. Let $R$ be the actual \textit{yield} of the applied policy. 
		The difference between $R$ and $R^*$, i.e. $R/R^*$, is a simple and effective metric to evaluate the policy. 
			For MARL algorithms, the convergence efficiency are also critical for practical effectiveness. It can be measured by the time consumed for converging to optimum and the time consumed for training per episode.

		\subsubsection{Comparing Methods} 
		
		\begin{enumerate} 
			\item \textbf{Contract First (CF):} A common impression allocation strategy that combines a contract bidding function and a potential contract shortfall tackling strategy. Based on the optimal bidding function defined by \eqref{eq:optimal_f}, when a contract shortfall risk is detected \citep{balseiro2014yield}, all remaining impressions are allocated to the contracts.
			
			\item \textbf{PID Controller (PID):} A widely used technique in display advertising \citep{zhang2016feedback} to fulfill contracts by even pacing. We adopt this technique to regulate $\alpha_j$ in Eq. \eqref{eq:optimal_f} to satisfy each contract $j$'s demand.
			
			\item \textbf{MADDPG:} Based on the optimal bidding function defined by \eqref{eq:optimal_f}, we implement the state-of-the-art MARL method \citep{lowe2017multi} to coordinate the agents with immediate reward function $r_j^{(t)}$ for impression allocation. 
			
			\item \textbf{MAPOLO:} Based on the optimal bidding function defined by \eqref{eq:optimal_f}, we implement our multi-agent policy optimization method for impression allocation, in which agents only receive local observations and the new reward function defined by \eqref{eq:r_def} is adopted. 
		\end{enumerate}
		
		\subsubsection{Implementation Details}
		
		We take a fully connected neural network with 3 hidden layers and 64 nodes for both actor and critic in each agent and another identically structured neural network as the new reward function \eqref{eq:r_def}. The mini-batch size is set to 32 and the replay memory size is set to 100,000. 
		The action range is limited to $[-0.1, 0.1]$ and the action noise is implemented by a normal distribution generator with mean $0$ and standard deviation $0.05$. Following the common practice of DDPG \citep{lillicrap2015continuous}, we set $\tau = 0.02$ to update target network parameters with that of actor and critic network, and the learning rate of actor and critic is set to $\num{1E-3}$ and $\num{1E-4}$ respectively.
		We also conducted experiments with different parameters. Larger action range and standard deviation of the noise would deteriorate the convergence efficiency due to larger exploration space, while tuning other parameters usually leads to similar results.

		\subsection{Evaluation Results}
		
		
		We conduct experiments to compare the performance of CF, PID, MADDPG and MAPOLO on all the five datasets. The initial parameter $\alpha_j^{(0)}$ is set to be the optimal one of training data. 
		As argued in section 1, due to the instability of auction environment, the impression volume and market price
		in the testing data deviate from that of training data.
		To present the performances of different approaches under the unstable environment, the instability statistics and experimental results based on testing dataset are summarized in Table \ref{table:yield_cmp}. 
		We can see that MAPOLO outperforms CF and PID in all datasets, and the overall improvement over CF and PID is $\num{5.7}$\% and $\num{4.5}$\% respectively. 
		It is worth noting that the result of MADDPG is not listed due to the convergence efficiency problem, which will be discussed in next subsection.
		
		\begin{table}
			\caption{The datasets statistics and $R/R^*$ of CF, PID and MAPOLO in 5 testing datasets.}
			\centering
			\scalebox{0.62}{
			\setlength\tabcolsep{3pt}
				\begin{tabular}{ccrrrrrrc}
					\toprule
					
					Publisher & Contracts & Impressions & \tabincell{c}{Contract\\Demands} & \tabincell{c}{Impression\\Difference} & \tabincell{c}{Market Price\\Difference} & CF &  PID  &   MAPOLO \\
					\hline
					1  & 68	& 4.9M 	& 2.7M 	& -6.4\% 	& -25.4\% 	& 0.92       &  	0.90	&  \textbf{0.93} \\
					2  & 7	& 2.9M 	& 1.3M 	& 4.8\% 	& 	38.7\%	& 	0.86	&  	0.89	&  \textbf{0.92}	 \\
					3  & 25	& 2.8M 	& 2.0M 	& -15.5\% 	& 52.5\% 	& 0.79  &  0.83	& 	\textbf{0.87} \\
					4  & 53	& 5.6M 	& 3.1M 	& 25.6\% 	& -30.9\% 	&  0.88	&  0.89 &  \textbf{0.96}	\\
					5  & 5	&  149K & 95K 	& -20.6\% 	&   16.9\%	& 0.89 &  0.91    & \textbf{0.92} \\
					\hline
					Average & & & & & & 0.87 & 0.88 & \textbf{0.92} \\
					\bottomrule
			\end{tabular}}
			\label{table:yield_cmp}
		\end{table}
		
		\begin{table}
			\caption{Detailed results on datasets of publisher 3 and 4.}
			\centering
			\scalebox{0.7}{
				\begin{tabular}{c|lrrrr}
					\toprule
					Publisher & Method & $R_{GC}$ & $R_{RTB}$ & $Q_{GC}$  & \textit{yield} \\ 
					\hline
					\multirow{4}{*}{3} & Optimal  &  6701.31   &  11796.01 & 2694.98 & 21192.30 \\
					& CF  		&  6771.85 &  6581.27  & 3406.18 & 16759.30\\
					& PID  		&  6771.85 &  8796.74  & 1925.04 & 17493.63\\
					& MAPOLO  	&  2999.01 &  13671.47 & 1734.29 & 18404.77\\
					\hline
					\multirow{4}{*}{4} & Optimal  &  5473.19  &   10609.10 & 1243.20 & 17325.71 \\
					& CF  		& 6037.51  &  7904.39 & 1251.33 & 15193.23\\
					& PID  		& 5950.39  & 8108.30  	& 1345.32 & 15404.01\\
					& MAPOLO  	& 5473.43  & 10185.15 & 1026.07 & 16684.65\\
					\bottomrule
			\end{tabular}}
			\label{table:ic_case}
		\end{table}
		
		To further investigate these methods' behaviors as the environment enormously changes, we present the detailed information of \textit{yield} for the representative datasets of publisher 3 and 4 in Table \ref{table:ic_case}, including $R_{GC}$, $R_{RTB}$ and $Q_{GC}$. Although CF achieves good $Q_{GC}$ in both datasets, it may result in poor \textit{yield} because of its belated adaptation to the environment change. Specifically, on dataset of publisher 3, when the competition become fiercer, CF still bids with the $\alpha$ obtained from yesterday, which makes the bid too low to win sufficient impressions at the beginning of the day. In order to satisfy the contract demands, all impressions, including those with high RTB revenue, are allocated to contracts when a shortfall risk is detected, therefore results in poor results of $R_{RTB}$. Similarly, on dataset for publisher 4, it bids with a relatively higher $\alpha$ and wins impressions with higher price, which results in poor $R_{RTB}$. 
		
		Compared with CF, PID shows the ability to cope with environment changes and obtains higher \textit{yield} on most datasets. PID adjusts $\alpha$ to pace the impression acquisition. MAPOLO is even better than PID due to the following two reasons. 
		First, it is critical for a PID strategy to be provided a proper target, which is the optimal impression volume to be allocated at each step. In practice, it is tricky and needs expert knowledge to be continuously optimized.
		Second, recall that the bids of contract bidders are determined by the allocation strategy and they are globally optimized to maximize the joint objective. .
		MAPOLO gives a proper solution to all these problems, and coordinates all contract bidders' behaviors to deliver satisfying results in an unstable environment.
		
		\subsection{Convergence Efficiency and Scalability}

		Convergence efficiency and scalability is an important factor that affects the practical value of MARL algorithms. Specifically, for impression allocation, it is of great necessity to accomplish the training process in an acceptable period of time before real application. In this section, we present experimental evaluations to investigate convergence efficiency and scalability of MADDPG and MAPOLO. 
		We compare both methods on two datasets: publisher 1 with 68 agents and publisher 3 with 25 agents, and the results are illustrated in Fig. \ref{fg:converging_process}.

		We can see that MAPOLO converges much faster than MADDPG and achieves satisfying $R/R^*$ of $\num{0.9}$ in both datasets. The reason why MAPOLO outperforms MADDPG can be interpreted from two aspects: first, the new reward function defined by Eq. \eqref{eq:r_def} is more "direct" and relieves the non-stationary environment problem; second, for MAPOLO, the input dimensionality for a single agent is not affected by the increasing agent number, while MADDPG is significantly suffered from that. When the agent number increases $\num{2.7}$ times ($\num{25}$ to $\num{68}$), the time consumption per episode for MADDPG increases $\num{4.7}$ times while MAPOLO only linearly increases $\num{2.5}$ times. This scalability advantage makes MAPOLO be applicable in scenario with more agents, especially in those large-scale industrial applications.
		
		\begin{figure} 
			\centering
			\includegraphics[width=0.48\textwidth]{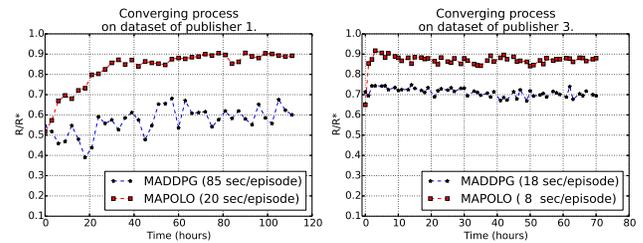}
			\caption{Converging processes of MADDPG and MAPOLO on dataset of publisher 1 and publisher 3. The x-axis is the elapsed time and the y-axis is the $R/R^*$. Average time consumption (in seconds) per episode of both methods are attached as the metric of convergence efficiency.
			}
			\label{fg:converging_process}
		\end{figure}
		
		\section{Related Work}
		
		Impression allocation, especially allocating between guaranteed contracts and RTB, is one of the most important issues to monetize traffic for a publisher. 
		\citep{ghosh2009bidding} was the first work considering contracts as bidders to compete with RTB, but the goal was to maximize the representativeness of contract impressions from the advertisers' perspective.
		\citep{chen2014dynamic} proposed a revenue maximization strategy based on allocating and pricing the future contract impressions. However, the allocation was determined in advance rather than through real-time auctions.
		\cite{balseiro2014yield} proposed learning a stochastic policy to maximize the \textit{yield} of a publisher based on the \textit{contract first} strategy, which has been proved not optimal in our experimental evaluations. 
		Recently, a new strategy maximizing the total revenue was proposed in \citep{jauvion2018optimal}. However, the challenge of environment instability was not discussed or addressed.
		
		Reinforcement learning (RL) has been widely studied in computational advertising, especially in bidding strategy optimization \citep{cai2017real, amin2012budget}. In MARL, methods based on single agent settings usually deliver poor performance \citep{matignon2012independent} due to the non-stationary environment \citep{colby2015counterfactual}. Although some Nash equilibrium algorithms are proposed to solve this problem, it is hard to be applied in real-world applications due to the computational complexity \citep{de2012polynomial}. \citep{yang2018mean} tries to address this challenge via leveraging action information from neighbor agents, but the concept of neighborhood is hard to be defined in the online display advertising context. Our approach follows the standard paradigm of centralized training with decentralized execution, which does not need to define any communication channels \citep{foerster2016learning, mordatch2017emergence}. \citep{lowe2017multi,foerster2017counterfactual} are similar works and both are applicable to our problem. We adopt \citep{lowe2017multi} for its simplicity and avoid the reward credit assignment problem in \citep{foerster2017counterfactual}.

		\section{Conclusion}
		
		In this paper, we proposed a MARL approach to maximize the total \textit{yield} of a publisher by allocating impressions between guaranteed contracts and RTB. We derived the optimal impression allocation strategy by solving the optimal bidding function when contracts are treated as bidders. In order to implement the strategy with the practical challenges such as environment instability, we propose an efficient MARL method, MAPOLO, to coordinate agents' behaviors in real-time. MAPOLO also solved the common challenges in MARL and achieved good convergence efficiency and scalability compared with the state-of-the-art MARL methods.

\end{CJK}
	
	\Urlmuskip=0mu plus 1mu\relax
	\bibliographystyle{aaai}
	\bibliography{bibliography}

\begin{thebibliography}{}

\bibitem[\protect\citeauthoryear{Amin \bgroup et al\mbox.\egroup
  }{2012}]{amin2012budget}
Amin, K.; Kearns, M.; Key, P.; and Schwaighofer, A.
\newblock 2012.
\newblock Budget optimization for sponsored search: Censored learning in mdps.
\newblock {\em arXiv preprint arXiv:1210.4847}.

\bibitem[\protect\citeauthoryear{Balseiro \bgroup et al\mbox.\egroup
  }{2014}]{balseiro2014yield}
Balseiro, S.~R.; Feldman, J.; Mirrokni, V.; and Muthukrishnan, S.
\newblock 2014.
\newblock Yield optimization of display advertising with ad exchange.
\newblock {\em Management Science} 60(12):2886--2907.

\bibitem[\protect\citeauthoryear{Cai \bgroup et al\mbox.\egroup
  }{2017}]{cai2017real}
Cai, H.; Ren, K.; Zhang, W.; Malialis, K.; Wang, J.; Yu, Y.; and Guo, D.
\newblock 2017.
\newblock Real-time bidding by reinforcement learning in display advertising.
\newblock In {\em Proceedings of the Tenth ACM International Conference on Web
  Search and Data Mining},  661--670.
\newblock ACM.

\bibitem[\protect\citeauthoryear{Chen \bgroup et al\mbox.\egroup
  }{}]{chencombining}
Chen, B.; Huang, J.; Huang, Y.; Kollias, S.; and Yue, S.
\newblock Combining guaranteed and spot markets in display advertising: selling
  guaranteed page views with stochastic demand.

\bibitem[\protect\citeauthoryear{Chen, Yuan, and Wang}{2014}]{chen2014dynamic}
Chen, B.; Yuan, S.; and Wang, J.
\newblock 2014.
\newblock A dynamic pricing model for unifying programmatic guarantee and
  real-time bidding in display advertising.
\newblock In {\em Proceedings of the Eighth International Workshop on Data
  Mining for Online Advertising},  1--9.
\newblock ACM.

\bibitem[\protect\citeauthoryear{Colby \bgroup et al\mbox.\egroup
  }{2015}]{colby2015counterfactual}
Colby, M.~K.; Kharaghani, S.; HolmesParker, C.; and Tumer, K.
\newblock 2015.
\newblock Counterfactual exploration for improving multiagent learning.
\newblock In {\em Proceedings of the 2015 International Conference on
  Autonomous Agents and Multiagent Systems},  171--179.
\newblock International Foundation for Autonomous Agents and Multiagent
  Systems.

\bibitem[\protect\citeauthoryear{De~Cote and Littman}{2012}]{de2012polynomial}
De~Cote, E.~M., and Littman, M.~L.
\newblock 2012.
\newblock A polynomial-time nash equilibrium algorithm for repeated stochastic
  games.
\newblock {\em arXiv preprint arXiv:1206.3277}.

\bibitem[\protect\citeauthoryear{Edelman, Ostrovsky, and
  Schwarz}{2007}]{edelman2007internet}
Edelman, B.; Ostrovsky, M.; and Schwarz, M.
\newblock 2007.
\newblock Internet advertising and the generalized second-price auction:
  Selling billions of dollars worth of keywords.
\newblock {\em American economic review} 97(1):242--259.

\bibitem[\protect\citeauthoryear{eMarketer}{}]{FisherPG}
eMarketer.
\newblock emarketer webinar: Programmatic advertising 2015 outlook.
\newblock
  \url{https://www.slideshare.net/eMarketerInc/emarketer-webinar-programmatic-advertising-2015-outlook}.

\bibitem[\protect\citeauthoryear{Foerster \bgroup et al\mbox.\egroup
  }{2016}]{foerster2016learning}
Foerster, J.; Assael, I.~A.; de~Freitas, N.; and Whiteson, S.
\newblock 2016.
\newblock Learning to communicate with deep multi-agent reinforcement learning.
\newblock In {\em Advances in Neural Information Processing Systems},
  2137--2145.

\bibitem[\protect\citeauthoryear{Foerster \bgroup et al\mbox.\egroup
  }{2017}]{foerster2017counterfactual}
Foerster, J.; Farquhar, G.; Afouras, T.; Nardelli, N.; and Whiteson, S.
\newblock 2017.
\newblock Counterfactual multi-agent policy gradients.
\newblock {\em arXiv preprint arXiv:1705.08926}.

\bibitem[\protect\citeauthoryear{Ghosh \bgroup et al\mbox.\egroup
  }{2009}]{ghosh2009bidding}
Ghosh, A.; McAfee, P.; Papineni, K.; and Vassilvitskii, S.
\newblock 2009.
\newblock Bidding for representative allocations for display advertising.
\newblock In {\em International Workshop on Internet and Network Economics},
  208--219.
\newblock Springer.

\bibitem[\protect\citeauthoryear{iab}{2017}]{iab-hy-2017}
2017.
\newblock Iab internet advertising revenue report.
\newblock
  \url{https://www.iab.com/wp-content/uploads/2018/05/IAB-2017-Full-Year-Internet-Advertising-Revenue-Report.REV_.pdf}.

\bibitem[\protect\citeauthoryear{Jauvion and
  Grislain}{2018}]{jauvion2018optimal}
Jauvion, G., and Grislain, N.
\newblock 2018.
\newblock Optimal allocation of real-time-bidding and direct campaigns.
\newblock In {\em Proceedings of the 24th ACM SIGKDD International Conference
  on Knowledge Discovery \& Data Mining},  416--424.
\newblock ACM.

\bibitem[\protect\citeauthoryear{Li \bgroup et al\mbox.\egroup
  }{2016}]{li2016optimal}
Li, J.; Ni, X.; Yuan, Y.; Qin, R.; and Wang, F.-Y.
\newblock 2016.
\newblock Optimal allocation of ad inventory in real-time bidding advertising
  markets.
\newblock In {\em Systems, Man, and Cybernetics (SMC), 2016 IEEE International
  Conference on},  003021--003026.
\newblock IEEE.

\bibitem[\protect\citeauthoryear{Lillicrap \bgroup et al\mbox.\egroup
  }{2015}]{lillicrap2015continuous}
Lillicrap, T.~P.; Hunt, J.~J.; Pritzel, A.; Heess, N.; Erez, T.; Tassa, Y.;
  Silver, D.; and Wierstra, D.
\newblock 2015.
\newblock Continuous control with deep reinforcement learning.
\newblock {\em arXiv preprint arXiv:1509.02971}.

\bibitem[\protect\citeauthoryear{Littman}{1994}]{littman1994markov}
Littman, M.~L.
\newblock 1994.
\newblock Markov games as a framework for multi-agent reinforcement learning.
\newblock In {\em Machine Learning Proceedings 1994}. Elsevier.
\newblock  157--163.

\bibitem[\protect\citeauthoryear{Lowe \bgroup et al\mbox.\egroup
  }{2017}]{lowe2017multi}
Lowe, R.; Wu, Y.; Tamar, A.; Harb, J.; Abbeel, O.~P.; and Mordatch, I.
\newblock 2017.
\newblock Multi-agent actor-critic for mixed cooperative-competitive
  environments.
\newblock In {\em Advances in Neural Information Processing Systems},
  6379--6390.

\bibitem[\protect\citeauthoryear{Matignon, Laurent, and
  Le~Fort-Piat}{2012}]{matignon2012independent}
Matignon, L.; Laurent, G.~J.; and Le~Fort-Piat, N.
\newblock 2012.
\newblock Independent reinforcement learners in cooperative markov games: a
  survey regarding coordination problems.
\newblock {\em The Knowledge Engineering Review} 27(1):1--31.

\bibitem[\protect\citeauthoryear{Mordatch and
  Abbeel}{2017}]{mordatch2017emergence}
Mordatch, I., and Abbeel, P.
\newblock 2017.
\newblock Emergence of grounded compositional language in multi-agent
  populations.
\newblock {\em arXiv preprint arXiv:1703.04908}.

\bibitem[\protect\citeauthoryear{Nisan \bgroup et al\mbox.\egroup
  }{2007}]{nisan2007algorithmic}
Nisan, N.; Roughgarden, T.; Tardos, E.; and Vazirani, V.~V.
\newblock 2007.
\newblock {\em Algorithmic game theory}.
\newblock Cambridge university press.

\bibitem[\protect\citeauthoryear{Sutton, Barto, and others}{}]{Sutton2017RL}
Sutton, R.~S.; Barto, A.~G.; et~al.
\newblock Reinforcement learning: An introduction.
\newblock \url{http://incompleteideas.net/book/bookdraft2017nov5.pdf}.

\bibitem[\protect\citeauthoryear{Tesauro}{2004}]{tesauro2004extending}
Tesauro, G.
\newblock 2004.
\newblock Extending q-learning to general adaptive multi-agent systems.
\newblock In {\em Advances in neural information processing systems},
  871--878.

\bibitem[\protect\citeauthoryear{Wu \bgroup et al\mbox.\egroup
  }{2018}]{wu2018budget}
Wu, D.; Chen, X.; Yang, X.; Wang, H.; Tan, Q.; Zhang, X.; Xu, J.; and Gai, K.
\newblock 2018.
\newblock Budget constrained bidding by model-free reinforcement learning in
  display advertising.
\newblock {\em arXiv preprint arXiv:1802.08365}.

\bibitem[\protect\citeauthoryear{Yang \bgroup et al\mbox.\egroup
  }{2018}]{yang2018mean}
Yang, Y.; Luo, R.; Li, M.; Zhou, M.; Zhang, W.; and Wang, J.
\newblock 2018.
\newblock Mean field multi-agent reinforcement learning.
\newblock {\em arXiv preprint arXiv:1802.05438}.

\bibitem[\protect\citeauthoryear{Yuan, Wang, and Zhao}{2013}]{yuan2013real}
Yuan, S.; Wang, J.; and Zhao, X.
\newblock 2013.
\newblock Real-time bidding for online advertising: measurement and analysis.
\newblock In {\em Proceedings of the Seventh International Workshop on Data
  Mining for Online Advertising}, ~3.
\newblock ACM.

\bibitem[\protect\citeauthoryear{Zhang \bgroup et al\mbox.\egroup
  }{2016}]{zhang2016feedback}
Zhang, W.; Rong, Y.; Wang, J.; Zhu, T.; and Wang, X.
\newblock 2016.
\newblock Feedback control of real-time display advertising.
\newblock In {\em Proceedings of the Ninth ACM International Conference on Web
  Search and Data Mining},  407--416.
\newblock ACM.

\end{thebibliography}
\end{document}